\documentclass{svproc}

\usepackage{latexsym, amssymb, amsfonts, amsmath, amsthm, graphics, graphicx, subfigure, bbm, stmaryrd, xcolor,url}
\usepackage{float}

\usepackage{graphicx}
\usepackage{verbatim}
\usepackage{appendix}

\theoremstyle{definition}

\theoremstyle{remark}

\numberwithin{equation}{section}

\begin{document}
\mainmatter
\bibliographystyle{plain}


\title{A cost-reducing partial labeling estimator in text classification problem}
\titlerunning{cost-reducing partial labelling estimator}

\author{Jiangning Chen\inst{1} \and Zhibo Dai\inst{2}\and Juntao Duan\and Qianli Hu\and Ruilin Li\and Heinrich Matzinger\and Ionel Popescu\and Haoyan Zhai}

\authorrunning{J.Chen, Z.Dai, J.Duan, Q.Hu, R.Li, H.Matzinger, I.Popescu, H.Zhai}
\institute{School of Mathematics, Georgia Institute of Technology, Altanta, GA 30313,
\email{jchen444@math.gatech.edu},
\and
School of Mathematics, Georgia Institute of Technology, Altanta, GA 30313,
\email{zdai37@gatech.edu}}

\maketitle 

\begin{abstract}
We propose a new approach to address the text classification problems when learning with partial labels is beneficial. Instead of offering each training sample a set of candidate labels, we assign negative-oriented labels to the ambiguous training examples if they are unlikely fall into certain classes. We construct our new maximum likelihood estimators with self-correction property, and prove that under some conditions, our estimators converge faster. Also we discuss the advantages of applying one of our estimator to a fully supervised learning problem. The proposed method has potential applicability in many areas, such as crowd-sourcing, natural language processing and medical image analysis.
\end{abstract}

\section{Introduction}

In some circumstances, the process of labeling is distributed among less-than-expert assessors. With the fact that some data may belong to several classes by nature, their labeling for hundreds of pictures, texts, or messages a day is error-prone. The invention of partial labeling seeks to remedy the labor: instead of assigning one or some exact labels, the annotators can offer a set of possible candidate solutions for one sample, thus providing a buffer against potential mistakes [1, 4, 8, 16, 17, 26]; Other partial labeling settings involve repeated labeling to filter out noises, or assessing the quality of the labelers [18,22] to enhance the reliability of the models.

As the data size in companies such as FANG(Facebook, Amazon, Netflix, Google) constantly reaches the magnitude of Petabyte, the demand for quick, yet still precise labeling is ever growing. Viewing some practices, the partial labeling frameworks that we know of exhibit some limitations. For instance, in a real-world situation concerning NLP, if the task is to determine the class/classes of one article, an annotator with a bachelor degree of American literature might find it difficult to determine if an article with words dotted with 'viscosity', 'gradient', and 'Laplacian' etc. belongs to computer science, math, physics, chemistry, or none of the classes above. As a result, the annotator might struggle within some limited amount of time amid a large pool of label classes and is likely to make imprecise choices even in a lenient, positive-oriented partial labeling environment. Another issue is the cost. Repeated labeling and keeping track of the performance of each labeler may be pricey, and the anonymity of the labelers can raise another barrier wall to certain partial labeling approaches.

Taking into consideration the real world scenarios, we present a new method to tackle the problem of how to gather at a large scale partially correct information from diverse annotators, while remaining efficient and budget-friendly. Still taking the above text classification problem as the example. Although that same annotator might not easily distinguish which categories the above-mentioned article belongs to, in a few seconds he/she can rule out the possibility the article is related to cuisines, TV-entertainment, or based on his/her own expertise, novels. In our partial label formulation, the safe choices, crossed-off categories labeled by annotators can still be of benefit. Furthermore, when contradictory labels are marked on one training sample and the identities of the labelers unknown, our introduced self-correcting estimator can select, and learn from the categories where the labels agree. 

Based on this, we propose a new way to formulate partial labeling. For some documents, instead of having the exact labels, not belonging to certain classes is the information provided. To make use of both kinds of data, we propose two maximum likelihood estimators, one of which has a self-correction property to estimate the distribution of each classes. By making both type of labeled data contribute in the training process, we proved that our estimators converge faster than traditional Naive Bayes estimator. We finally find a way to apply our method to some only positive labeled data set, which is identified as a fully supervised learning problem, and achieve a better result compared to the traditional Naive Bayes. 

The rest of this paper is organized as follows. Section 2 introduces the related works about text classification and partial labeling. Section 3 introduces the formulation of this problem, and conclude the result of traditional Naive Bayes estimator. Section 4 introduces the main results of our estimator, as well as how to apply it in fully supervised learning problem. Section 5 reports experimental results of comparative studies. Finally, Section 6 concludes the paper and discusses future research issues.

\section{Related work}
The text classification problem is seeking a way to best distinguish different types of documents\cite{dumais1998inductive,larkey1998automatic}. Being a traditional natural language processing problem, one needs to make full use of the words and sentences, converting them into various input features, and applying different models to process training and testing. A common way to convert words into features is to encoding them based on the term frequency and inverse document frequency, as well as the sequence of the words. There are many results about this, for example, tf-idf\cite{ramos2003using} encodes term $t$ in document $d$ of corpus $D$ as:
$$tfidf(t,d,D) = tf(t,d)\cdot idf(t,D),$$ where $tf(t,d)$ is defined as term frequency, it can be computed as $tf(t,d) = \frac{|{t:t\in d}|}{|d|}$, and $idf(t,D)$ is defined as inverse document frequency, it can be computed as $idf(t,D) = \log \frac{|D|}{|\{d\in D:t\in d\}|}$. We also have n-gram techniques, which first combines $n$ nearest words together as a single term, and then encodes it with tf-idf. Recently, instead of using tf-idf, \cite{schneider2004new} defines a new feature selection score for text classification based on the KL-divergence between the distribution of words in training documents and their classes.

A popular model to achieve our aim is to use Naive Bayes model\cite{friedman1997bayesian,langley1992analysis}, the label for a given document $d$ is given by:
\begin{equation*}
    label(d) = \operatorname*{argmax}_j P(C_j)P(d|C_j),
\end{equation*} where $C_j$ is the $j$-th class. For example, we can treat each class as a multinomial distribution, and the corresponding documents are samples generated by the distribution. With this assumption, we desire to find the centroid for every class, by either using the maximum likelihood function or defining other different objective functions\cite{chen2018centroid} in both supervised and unsupervised learning version\cite{hofmann1999probabilistic}. Although the assumption of this method is not exact in this task, Naive Bayes achieves high accuracy in practical problems.

There are also other approaches to this problem, one of which is simply finding linear boundaries of classes with support vector machine\cite{joachims1998text,cortes1995support}. Recurrent Neural Network (RNN)\cite{liu2016recurrent,tang2015document} combined with word embedding is also a widely used model for this problem.

In real life, one may have different type of labels\cite{li2003learning}, in which circumstance, semi-supervised learning or partial-label problems need to be considered \cite{cour2011learning}. There are several methods to encode the partial label information into the learning framework. For the partial label data set, one can define a new loss combining all information of the possible labels, for example, in \cite{nguyen2008classification}, the authors modify the traditional $L^2$ loss
\[
L(w)=\frac{1}{n+m}\left[\sum_{i=1}^nl(x_i,y_i,w)+\sum_{i=1}^ml(x_i,Y_i,w)\right],
\]
where $Y_i$ is the possible label set for $x_i$ and $l(x_i,Y_i,w)$ is a non-negative loss function, and in \cite{cour2011learning}, they defined convex loss for partial labels as:
\[
L{_\Psi}(g(x),y) = \Psi(\frac{1}{|y|}{\sum_{a\in y}g_a(x)})+\sum_{a\notin y}\Psi(-g_a(x)),
\]
where $\Psi$ is a convex function, $y$ is a singleton, and $g_a(x)$ is a score function for label $a$ as input $x$.
A modification of the likelihood function is as well an approach to this problem and \cite{jin2003learning} gives the following optimization problem using Naive Bayes method
\[
\theta^*=\arg\max_\theta\sum_i\sum_{y_i\in S_i}p(y|x_i,\theta)
\]
where $S_i$ is the possible labels for $x_i$.

Meanwhile, the similarity of features among data could be considered to give a confidence of each potential labels for a certain data. In \cite{zhang2016partial}, K nearest neighbor (KNN) is adopted to construct a graph structure with the information of features.  While in \cite{li2003learning} Rocchio and K-means clustering are used.

\section{Formulation}
\subsection{General Setting}
Consider a classification problem with independent sample $x\in S$ and class set $C$, where $C = \{C_1,C_2,...,C_k\}$. We are interested in finding our estimator: $$\hat{y} = f(x;\theta) = (f_1(x;\theta),f_2(x;\theta),...,f_k(x;\theta))$$ for $y$, where $\theta = \{\theta_1,\theta_2,...,\theta_m\}$ is the parameter, and $f_i(x;\theta)$ is the likelihood function of sample $x$ in class $C_i$. Now assuming that in training set, we have two types of dataset $S_1$ and $S_2$, such that $S = S_1\cup S_2$: 
\begin{enumerate}

\item dataset $S_1$: we know exactly that sample $x$ is in a class, and not in other classes. In this case, define: $y = (y_1,y_2,...,y_k)$, if $x$ is in class $C_i$, then $y_i=1$. Notice that if this is a single label problem, then we have: $\sum_{i = 1}^k y_i = 1$.

\item dataset $S_2$: we only have the information that sample $x$ is not in a class, then $y_i=0$. In this case, define: $z = (z_1,z_2,...,z_k)$, if $x$ is not in class $C_i$, we have $z_i=1$.

\end{enumerate}

To build the model, we define the following likelihood ratio function and lihoodhood function:

\begin{equation}\label{likelihood_function}
    L_1(\theta) = \prod_{x\in S_1}\prod_{i=1}^k f_i(x;\theta)^{y_i}\prod_{x\in S_2}\prod_{i=1}^k f_i(x;\theta)^{\frac{1-z_i}{k-\sum_{j\neq i}z_j}}.
\end{equation}

\begin{equation}\label{likelihood_ratio_function}
    L_2(\theta) = \prod_{x\in S}\frac{\prod_{i=1}^k f_i(x;\theta)^{y_i(x)+t}}{\prod_{i=1}^k f_i(x;\theta)^{z_i(x)}}=\prod_{x\in S}{\prod_{i=1}^k f_i(x;\theta)^{y_i(x)-z_i(x)+t}}.
\end{equation}
The $t$ in $L_2$ satisfy $t>1$, which is a parameter to avoid non-convexity. 

The intuition of $L_1$ is to consider the sample labeled $z_i=1$ has equal probability to be labeled in the other classes, each of the classes will have probability $\frac{1-z_i}{k-\sum_{j\neq i}z_j}.$ And the intuition of $L_2$ is to consider this in a likelihood ratio way, the $z_i=1$ labeled sample will have negative affection for class $C_i$, so we put it in the denominator. With $t>1$, all the terms in denominator will finally be canceled out, so that even $f_i(x;\theta)=0$ for some sample $x\in S$ will not cause trouble. Another intuition for $L_2$ is that, it can be self-correct the repeated data, which has been labeled incorrectly.

Take logarithm for both side, we obtain the following functions:
\begin{equation}\label{log_likelihood}
    \log(L_1(\theta)) = \sum_{x\in S_1} \sum_{i = 1}^k y_i(x)\log{f_i(x,\theta)} + \sum_{x\in S_2} \sum_{i = 1}^k {\frac{1-z_i}{k-\sum_{j\neq i}z_j}}\log{f_i(x,\theta)},
\end{equation}
and
\begin{equation}\label{log_likelihood_ratio}
    \log(L_2(\theta)) = \sum_{x\in S} \sum_{i = 1}^k (y_i(x)+t-z_i(x))\log{f_i(x,\theta)}.
\end{equation}

We would like to find our estimator $\hat{\theta}$ such that \eqref{log_likelihood_ratio} or \eqref{log_likelihood} reaches maximum.



\subsection{Naive Bayes classifier in text classification problem}
For Naive Bayes model. Let class $C_i$ with centroid $\theta_i = (\theta_{i_1},\theta_{i_2},...,\theta_{i_v}),$
where $v$ is the total number of the words and $\theta_i$ satisfies: $\sum_{j=1}^v \theta_{i_j} = 1$. Assuming independence of the words, the most likely class for a document $d = (x_1,x_2,...,x_{v})$ is computed as:

\begin{eqnarray}\label{naive_bayes}
label(d) &=& \operatorname*{argmax}_i P(C_i)P(d|C_i)\\
         &=& \operatorname*{argmax}_i P(C_i)\prod_{j=1}^{v} (\theta_{i_j})^{x_j}\nonumber\\
         &=& \operatorname*{argmax}_i \log{P(C_i)} + \sum_{j=1}^{v} x_j\log{\theta_{i_j}}\nonumber.
\end{eqnarray}

So we have:
\begin{equation*}
    \log f_i(d,\theta) = \log{P(C_i)} + \sum_{j=1}^{v} x_j\log{\theta_{i_j}}.
\end{equation*}

For a class $C_i$, we have the standard likelihood function: 
\begin{equation}\label{nb_likelihood}
L(\theta) = \prod_{x\in C_i}\prod_{j=1}^v \theta_{i_j}^{x_j},
\end{equation}
Take logarithm for both side, we obtain the log-likelihood function:
\begin{equation}\label{nb_log_likelihood}
\log{L(\theta)} = \sum_{x\in C_i}\sum_{j=1}^v x_j\log{\theta_{i_j}}.
\end{equation}

We would like to solve optimization problem:
\begin{eqnarray}\label{nb_optimal_prob}
\max\ & &L(\theta)\\
\text{subject to}: & &\sum_{j = 1}^v \theta_{i_j} = 1 \nonumber \\
& & \theta_{i_j}\geq 0.
\end{eqnarray}

The problem \eqref{nb_optimal_prob} can be solve explicitly with \eqref{nb_log_likelihood} by Lagrange Multiplier, for class $C_i$, we have $\theta_{i} = \{\theta_{i_1},\theta_{i_2},...,\theta_{i_v}\}$, where:
\begin{equation}\label{nb_estimator}
\hat{\theta}_{i_j} = \frac{\sum_{d\in C_i}x_j}{\sum_{d\in C_i}\sum_{j=1}^v x_j}.
\end{equation}

For estimator $\hat{\theta}$, we have following theorem.

\begin{theorem}\label{nb_property}
Assume we have normalized length of each document, that is: $\sum_{j=1}^v x_j = m$ for all $d$, the estimator \eqref{nb_estimator} satisfies following properties:
\begin{enumerate}
    \item   
    $\hat{\theta}_{i_j}$ is unbiased.
    \item   
    $E[|\hat{\theta}_{i_j}-\theta_{i_j}|^2] = \frac{\theta_{i_j}(1-\theta_{i_j})}{|C_i|m}$.
\end{enumerate}
\end{theorem}
The proof of this theorem can be found in appendix.

\section{Main Result}
From Theorem.\ref{nb_property}, we can see that traditional Naive Bayes estimator $\hat{\theta}$ is an unbiased estimator with variance $O(\frac{\theta_{i_j}(1-\theta_{i_j})}{|C_i|m})$. Now we are trying to solve our estimators, and prove they can use the data in dataset $S_2$, and perform better than traditional Naive Bayes estimator.

\subsection{Text classification with $L_1$ setting \eqref{likelihood_function}} In order to use data both in $S_1$ and $S_2$, we would like to solve \eqref{nb_optimal_prob} with $L(\theta) = L_1(\theta)$, where $L_1$ is defined as \eqref{likelihood_function}, let:
$$G_i = 1 - \sum_{j=1}^v \theta_{i_j},$$ by Lagrange multiplier, we have: 

\begin{equation*}
\left\{
\begin{aligned}
&\frac{\partial \log(L_1)}{\partial \theta_{i_j}}+\lambda_i\frac{\partial G_i}{\partial \theta_{i_j}}=0\  \forall\ 1\leq i\leq k\text{ and } \forall\ 1\leq j\leq v\\
&\sum_{j=1}^v \theta_{i_j} = 1,\ \forall\ 1\leq i\leq k
\end{aligned}
\right.
\end{equation*}

Plug in, we obtain:
\begin{equation}\label{L_1_nb_solutions}
\left\{
\begin{aligned}
&\sum_{x\in S_1}\frac{y_i(x)x_j}{\theta_{i_j}} + \sum_{x\in S_2}\frac{1-z_i(x)}{k-\sum_{l\neq i}z_l(x)}\cdot\frac{x_j}{\theta_{i_j}} - \lambda_i = 0\\
&\sum_{j=1}^v \theta_{i_j} = 1,\ \forall\ 1\leq i\leq k
\end{aligned}
\right.
\end{equation}
Here we take $i$ and $j$ to be $\forall\ 1\leq i\leq k\text{ and } \forall\ 1\leq j\leq v $

Solve \eqref{L_1_nb_solutions}, we got the solution of optimization problem \eqref{nb_optimal_prob}:
\begin{equation}\label{L_1_optimal_solutions}
\hat{\theta}_{i_j}^{L_1} = \frac{\sum_{x\in S_1}y_i(x)x_j+\sum_{x\in S_2}\frac{1-z_i(x)}{k-\sum_{l\neq i}z_l(x)}x_j}{\sum_{x\in S_1}y_i(x)\sum_{j=1}^v x_j+\sum_{x\in S_2}\frac{1-z_i(x)}{k-\sum_{l\neq i}z_l(x)}\sum_{j=1}^v x_j}.
\end{equation}

\begin{theorem}\label{L_1_property}
Assume we have normalized length of each document, that is: $\sum_{j=1}^v x_j = m$ for all $d$. Let $Z_i(x) = \frac{1-z_i(x)}{k-\sum_{l\neq i}z_l(x)}=K$, $l_{i_j}=E[x_j|Z_i=K]/m$. Assume further that $|\{i:z_i(x)=1\}|=K$ to be a constant for all $x\in S_2$, the estimator \eqref{L_1_optimal_solutions} satisfies following properties:
\begin{enumerate}
    \item   
    $\hat{\theta}_{i_j}^{L_1}$ is biased with
    \[
    E[\hat{\theta}_{i_j}^{L_1}-\theta_{i_j}]=\frac{|R_i|K(l_{i_j}-\theta_{i_j})}{|C_i|+|R_i|K}.
    \]
    \item   
    $E[|\hat{\theta}_{i_j}^{L_1}-\theta_{i_j}|^2] = O\left(\frac{1}{|S_1|+|S_2|}\right)$.
\end{enumerate}
\end{theorem}

\begin{proof}
\begin{enumerate}
\item We denote $Z_i(x) = \frac{1-z_i(x)}{k-\sum_{l\neq i}z_l(x)}=K$, $l_{i_j}=E[x_j|Z_i=K]/m$ and $R_i=\{x:z_i(x)=0\}$, we have:
\[
\hat{\theta}_{i_j}^{L_1} = \frac{\sum_{x\in S_1}y_i(x)x_j+\sum_{x\in S_2}Z_i(x)x_j}{(\sum_{x\in S_1}y_i(x)+\sum_{x\in S_2}Z_i(x))m} = \frac{\sum_{x\in S_1}y_i(x)x_j+\sum_{x\in S_2}Z_i(x)x_j}{(|C_i|+|R_i|K)m}
\]

Moreover, assuming that $p_i=P(y_i(x)=1)=|C_i|/|S_1|,q_i=P(z_i(x)=0)=|R_i|/|S_2|$, it holds that
\begin{align*}
E[\hat{\theta}_{i_j}^{L_1}]&=\frac{1}{(|C_i|+|R_i|K)m}\left(\sum_{x\in S_1}E[y_i(x)x_j]+\sum_{x\in S_2}E[Z_i(x)x_j]\right)\\
&=\frac{1}{(|C_i|+|R_i|K)m}\sum_{x\in S_1}p_iE[x_j|y_i(x)=1]\\
&+\frac{1}{(|C_i|+|R_i|K)m}\sum_{x\in S_2}q_iKE[x_j|Z_i(x)=K]\\
&=\frac{|C_i|E[x_j|y_i(x)=1]+|R_i|KE[x_j|Z_i(x)=K]}{(|C_i|+|R_i|K)m}\\
&=\frac{|C_i|\theta_{i_j}m+|R_i|Kl_{i_j}m}{(|C_i|+|R_i|K)m}.
\end{align*}
Thus,
\begin{equation*}
E[\hat{\theta}_{i_j}^{L_1}-\theta_{i_j}]=\frac{|R_i|K(l_{i_j}-\theta_{i_j})}{|C_i|+|R_i|K}
\end{equation*}

\item As is for the second part, we have
\begin{align*}
\left(\hat{\theta}_{i_j}^{L_1}\right)^2=\frac{1}{(|C_i|+|R_i|K)m}\Big(&\sum_{\alpha\in S_1}\sum_{\beta\in S_1}y_i(\alpha)y_i(\beta)\alpha_j\beta_j\\
&\sum_{\alpha\in S_2}\sum_{\beta\in S_2}Z_i(\alpha)Z_i(\beta)\alpha_j\beta_j\\
&\sum_{\alpha\in S_1}\sum_{\beta\in S_2}y_i(\alpha)Z_i(\beta)\alpha_j\beta_j\Big).
\end{align*}
Then, by introducing $C=(|C_i|+|R_i|K)m$ and $L_{i_j}=E[x_j^2|Z_i(x)=K]$ it is true that
\begin{align*}
E\left[\left(\hat{\theta}_{i_j}^{L_1}\right)^2\right]=&\frac{1}{C^2}
\Big(&&\sum_{x\in S_1}E[y_i^2(x)x_j^2]+\sum_{\alpha,\beta\in S_1,\alpha\neq\beta}E[y_i(\alpha)\alpha_j]E[y_i(\beta)\beta_j]\\
&&&+\sum_{x\in S_2}E[Z_i^2(x)x_j^2]+\sum_{\alpha,\beta\in S_2,\alpha\neq\beta}E[Z_i(\alpha)\alpha_j]E[Z_i(\beta)\beta_j]\\
&&&+2\sum_{\alpha\in S_1,\beta\in S_2}E[y_i(\alpha)\alpha_j]E[Z_i(\beta)\beta_j]\Big)\\
=&\frac{1}{C^2}\Big(&&|C_i|m\theta_{i_j}(1-\theta_{i_j}+m\theta_{i_j})+\left(|S_1|^2-|S_1|\right)p_i^2m^2\theta_{i_j}^2\\
&&&+|R_i|K^2L_{i_j}+\left(|S_2|^2-|S_2|\right)K^2q_i^2m^2l_{i_j}^2\\
&&&+2|C_i||R_i|m^2K\theta_{i_j}l_{i_j}|\Big)\\
=&\frac{1}{C^2}\Big(&&|C_i|m\theta_{i_j}(1-\theta_{i_j}+m\theta_{i_j})-|S_1|p_i^2m^2\theta_{i_j}^2\\
&&&+|R_i|K^2L_{i_j}-|S_2|K^2q_i^2m^2l_{i_j}^2\Big)\\
&&&+\left(\frac{|C_i|\theta_{i_j}m+|R_i|Kl_{i_j}m}{(|C_i|+|R_i|K)m}\right)^2
\end{align*}
Using the fact that $E\left|\hat{\theta}_{i_j}^{L_1}-E[\hat{\theta}_{i_j}^{L_1}]\right|^2=E\left[\left(\hat{\theta}_{i_j}^{L_1}\right)^2\right]-\left(E[\hat{\theta}_{i_j}^{L_1}]\right)^2$, we can conclude that
\begin{align*}
&E\left|\hat{\theta}_{i_j}^{L_1}-E[\hat{\theta}_{i_j}^{L_1}]\right|^2\\
&=\frac{1}{(|C_i|+|R_i|K)^2m^2}&&\Big(|C_i|m\theta_{i_j}(1-\theta_{i_j}+m\theta_{i_j})-|S_1|p_i^2m^2\theta_{i_j}^2\\
&&&+|R_i|K^2L_{i_j}-|S_2|K^2q_i^2m^2l_{i_j}^2\Big)\\
&=O\left(\frac{1}{|S_1|+|S_2|}\right)
\end{align*}
\end{enumerate}
\end{proof}

Comparing $\hat{\theta}_{i_j}$ and $\hat{\theta}_{i_j}^{L_1}$, we can see that even though our estimator is biased, we can see that the variance of $\hat{\theta}_{i_j}^{L_1}$ is significant smaller than the variance of $\hat{\theta}_{i_j}$, which means by using negative sample set, $\hat{\theta}_{i_j}^{L_1}$ converges way faster than original Naive Bayes estimator $\hat{\theta}_{i_j}$.

\subsection{Text classification with $L_2$ setting \eqref{likelihood_ratio_function}} Another way to use both $S_1$ and $S_2$ dataset is to solve \eqref{nb_optimal_prob} with $L(\theta) = L_2(\theta)$, where $L_2$ is defined as \eqref{likelihood_ratio_function}, let:
$$G_i = 1 - \sum_{j=1}^v \theta_{i_j},$$ by Lagrange multiplier, we have: 

\begin{equation*}
\left\{
\begin{aligned}
&\frac{\partial \log(L_2)}{\partial \theta_{i_j}}+\lambda_i\frac{\partial G_i}{\partial \theta_{i_j}}=0\  \forall\ 1\leq i\leq k\text{ and } \forall\ 1\leq j\leq v\\
&\sum_{j=1}^v \theta_{i_j} = 1,\ \forall\ 1\leq i\leq k
\end{aligned}
\right.
\end{equation*}

Plug in, we obtain:
\begin{equation}\label{L_2_nb_solutions}
\left\{
\begin{aligned}
&\sum_{x\in S}(y_i(x)+t-z_i(x))\frac{x_j}{\theta_{i_j}}-\lambda_i=0\  \forall\ 1\leq i\leq k\text{ and } \forall\ 1\leq j\leq v\\
&\sum_{j=1}^v \theta_{i_j} = 1,\ \forall\ 1\leq i\leq k
\end{aligned}
\right.
\end{equation}

Solve \eqref{L_2_nb_solutions}, we got the solution of optimization problem \eqref{nb_optimal_prob}:
\begin{equation}\label{L_2_optimal_solutions}
\hat{\theta}_{i_j}^{L_2} = \frac{\sum_{x\in S} (y_i(x)+t-z_i(x))x_j}{\sum_{j=1}^v \sum_{x\in S} (y_i(x)+t-z_i(x))x_j}.
\end{equation}

Notice that the parameter $t$ here is used to avoid non-convexity, when $0\leq t<1$, the optimization problem \eqref{nb_optimal_prob} has the optimizer located on the boundary of $\theta$, which cannot be solved explicitly.

\begin{theorem}\label{L_2_property}
Assume we have normalized length of each document, that is: $\sum_{j=1}^v x_j = m$ for all $d$. Let $|C_{i}|$ denote the number of documents in Class $i$ and $|D_{i}|$ denote the number of documents labelled not in Class $i$ with $p_{i}=\frac{|C_{i}|}{|S|}$ and $q_{i}=\frac{|D_{i}|}{|S|}$. Further, we assume if a document $x$ is labelled not in Class $i$, it will have equal probability to be in any other class. Then the estimator \eqref{L_2_optimal_solutions} satisfies following properties:
\begin{enumerate}
    \item   
    $\hat{\theta}_{i_j}^{L_2}$ is biased and $E[\hat{\theta}_{i_j}^{L_2}-\theta_{i_j}] = O(t)$.
    \item   
    $var[\hat{\theta}_{i_j}^{L_2}]=O(\frac{1}{m(|S_1|+|S_2|)}).$

\end{enumerate}
\end{theorem}

\begin{proof}
First of all, we can simplify \eqref{L_2_optimal_solutions} using our assumption to be
\begin{equation*}
    \hat{\theta}_{i_j}^{L_2} = \frac{\sum_{x\in S} (y_i(x)+t-z_i(x))x_j}{\sum_{x\in S} (y_i(x)+t-z_i(x))m}.
\end{equation*}
For each $x\in S$, $\theta_{i_j}$ follows multinomial distribution. For $x\in C_{i}\subset S_{1}$, $E[x_j]=m\theta_{i_j}$ and $E[x_j^2] = m\theta_{i_j}(1-\theta_{i_j}+m\theta_{i_j})$. For $x\in S_{2}$ with $z_{i}\left(x\right)=1$, which means $x$ is labelled not in Class $i$, we have $E[x_{j}]=\frac{m\sum_{l\neq i}\theta_{l_j}}{k-1}$ and $E[x_{j}^2]=\frac{m\sum_{l\neq i}\theta_{l_j}\left(1-\theta_{l_j}\right)}{(k-1)^2}+\frac{m^2\left(\sum_{l\neq i}\theta_{l_j}\right)^2}{(k-1)^2}$. Therefore, we have the following properties:

Moreover, we let $M_{l}=\sum_{i=1}^{k}\theta_{i_j}-\theta_{l_j}$ and $N=m\sum_{x\in S}\left(y_{i}(x)+t-z_i(x)\right)=m\left(|C_{i}|-|D_{i}|+t|S|\right)$.
\begin{enumerate}
    \item \label{L_2_biased}The expectation of the statistical quantity is
    \begin{align*}
    &E[\hat{\theta}_{i_j}^{L_2}]\\
    =&\frac{\sum_{x\in S} (y_i(x)+t-z_i(x))E[x_j]}{m\left(|C_{i}|-|D_{i}|+t|S|\right)}\\
    =&\frac{t\sum_{x\in S_{1}}E\left(x_{j}\right)+t\sum_{x\in S_{2}}E\left(x_{j}\right)+m|C_{i}|\theta_{i_j}-m|D_{i}|\frac{\sum_{l\neq i}\theta_{l_j}}{k-1}}{m\left(|C_{i}|-|D_{i}|+t|S|\right)}\\
    =&\frac{t\sum_{l=1}^{k}|C_{l}|\theta_{l_j}+|C_{i}|\theta_{i_j}+t\sum_{l=1}^{k}\frac{|D_{l}|M_{l}}{k-1}-|D_{i}|\frac{M_{i}}{k-1}}{\left(|C_{i}|-|D_{i}|+t|S|\right)}\\
    =&\frac{t\sum_{l=1}^{k}|p_{l}|\theta_{l_j}+|p_{i}|\theta_{i_j}+t\sum_{l=1}^{k}\frac{|q_{l}|M_{l}}{k-1}-|q_{i}|\frac{M_{i}}{k-1}}{|p_{i}|-|q_{i}|+t}
    \end{align*}
    Therefore, we have the Biase: $E[\hat{\theta}_{i_j}^{L_2}-\theta_{i_j}]$ is in order $O\left(t\right)$, which gives us the desired result.

    \item We use the property of variance to decompose our estimation into two parts:
    \begin{equation*}
        E[|\hat{\theta}_{i_j}^{L_2}-E[\hat{\theta}_{i_j}^{L_2}]|^2]=E[(\hat{\theta}_{i_j}^{L_2})^2]-\left(E[\hat{\theta}_{i_j}^{L_2}]\right)^{2}.
    \end{equation*}
    Therefore, we have:
    \begin{align*}\label{L_2_var}
    &var[\hat{\theta}_{i_j}^{L_2}]\\
    =&E[(\hat{\theta}_{i_j}^{L_2})^2]-\left(E[\hat{\theta}_{i_j}^{L_2}]\right)^2\\
    =&E\left(\frac{\sum_{x\in S}\left(y_{i}(x)+t-z_i(x)\right)^2x_{j}^2}{N^2}\right)\\
    +&E\left(\frac{\sum_{x^1,x^2\in S,x^1\neq x^2}(y_{i}(x^1)+t-z_i(x^1))(y_{i}(x^2)+t-z_i(x^2))x^1_j x^2_j}{N^2}\right)\\
    -&\frac{(\sum_{x\in S}(y_{i}(x)+t-z_i(x))E[x_j])^2}{N^2}\\
    =&\frac{\sum_{x\in S}(y_{i}(x)+t-z_i(x))^2var(x_j)}{N^2}\\
    +&\frac{\sum_{x^1,x^2\in S,x^1\neq x^2}(y_{i}(x^1)+t-z_i(x^1))(y_{i}(x^2)+t-z_i(x^2))cov(x^1_j,x^2_j)}{N^2}.
    \end{align*}
    
    In terms of the order,  we can take the $cov(x_j^1,x_j^2) = 0$ as $x_1$ and $x_2$ are independent documents.  Therefore:
    \begin{align*}
        &var[\hat{\theta}_{i_j}^{L_2}]\\
        =&\frac{\sum_{x\in S}(y_{i}(x)+t-z_i(x))^2var(x_j)}{N^2}\\
        =&\frac{\sum_{x\in S}y_{i}(x)var(x_j)+\sum_{x\in S}z_{i}(x)var(x_j)+t^2\sum_{x\in S}var(x_j)}{N^2}\\
        &+\frac{2t\sum_{x\in S}y_{i}(x)var(x_j)-2t\sum_{x\in S}z_{i}(x)var(x_j)}{N^2}\\
        =&\frac{(1+2t)|C_{i}|\theta_{i_j}(1-\theta_{i_j})+(1-2t)\frac{|D_{i}|}{(k-1)^2}\sum_{l\neq i}\theta_{l_j}(1-\theta_{l_j})}{m\left(|C_{i}|-|D_{i}|+t|S|\right)^2}\\
        &+\frac{t^2\sum_{l=1}^{k}|C_{l}|\theta_{l_j}(1-\theta_{l_j})}{m\left(|C_{i}|-|D_{i}|+t|S|\right)^2}+\frac{t^2\sum_{l=1}^{k}\frac{|D_{l}|}{(k-1)^2}\sum_{b\neq l}\theta_{b_j}(1-\theta_{b_j})}{m\left(|C_{i}|-|D_{i}|+t|S|\right)^2}\\
        =&\frac{(1+2t)p_{i}\theta_{i_j}(1-\theta_{i_j})+(1-2t)\frac{q_{i}}{(k-1)^2}\sum_{l\neq i}\theta_{l_j}(1-\theta_{l_j})}{m|S|\left(p_{i}-q_{i}+t\right)^2}\\
        &+\frac{t^2\sum_{l=1}^{k}p_{l}\theta_{l_j}(1-\theta_{l_j})}{m|S|\left(p_{i}-q_{i}+t\right)^2}+\frac{t^2\sum_{l=1}^{k}\frac{q_{l}}{(k-1)^2}\sum_{b\neq l}\theta_{b_j}(1-\theta_{b_j})}{m|S|\left(p_{i}-q_{i}+t\right)^2}\\
        =&O(\frac{1}{|S_1|+|S_2|}).
    \end{align*}
    \end{enumerate}
    \end{proof}
    Using the same strategy as in \ref{L_2_biased}, we have the first part of our variance estimation should be of order $O(\frac{1}{m|S|})$, which is less than the order of variance for Naive Bayes estimation:$O(\frac{1}{|C_i|})$. We also showed that its order is $O(\frac{1}{|S_1|+|S_2|}) < O(\frac{1}{|C_{i}|})$,
    therefore, $\hat{\theta}_{i_j}^{L_2}$ converges faster than $\hat{\theta}_{i_j}$.

\subsection{Improvement of Naive Bayes estimator with only $S_1$ dataset} Now assume that we don't have dataset $S_2$, but only have dataset $S=S_1$, can we still do better than traditional Naive Bayes estimator $\hat{\theta}$? To improve the estimator, we can try to use $L_1$ or $L_2$ setting. With $z(x) = 1-y(x)$, we can define function $z$ on $S_1$ dataset. In this setting, we have actually defined $S_2=S_1.$

With simple computation, we have the estimator of $L_1$ is the same as  $\hat{\theta}_{i_j}$. as for the estimator for ${L_2}$, we have:
\begin{equation}\label{our_method_estimator}
    \hat{\theta}_{i_j}^* = \frac{\sum_{x\in S} (2y_i(x)+t-1)x_j}{\sum_{j=1}^v \sum_{x\in S} (2y_i(x)+t-1)x_j},
\end{equation}
and by Theorem \ref{L_2_property}, we have:
\begin{corollary}\label{our_method_property}
Assume we have normalized length of each document, that is: $\sum_{j=1}^v x_j = m$ for all $d$. With only dataset $S_1$, let $S_2=S_1$, define $z(x)=1-y(x)$, Then the estimator \eqref{our_method_estimator} satisfies following properties:
\begin{enumerate}
    \item   
    $\hat{\theta}_{i_j}^{*}$ is biased, $E[\hat{\theta}_{i_j}^{*}-\theta_{i_j}] = O(t)$.
    \item   
    $E[|\hat{\theta}_{i_j}^{*}-\theta_{i_j}|^2] = O(\frac{1}{|S|}).$
\end{enumerate}
\end{corollary}

\section{Experiment}
We applied our method on top 10 topics of single labeled documents in Reuters-21578 data\cite{reuters_data}, and 20 news group data\cite{20_news}. we compare the result of traditional Naive Bayes estimator $\hat{\theta}_{i_j}$ and our estimator $\hat{\theta}_{i_j}^{L_1}$, $\hat{\theta}_{i_j}^{L_2}$, as well as $\hat{\theta}_{i_j}^{*}$. $t$ is chosen to be $2$ in all the following figures. The data in $S_2$ is generated randomly by not belong to a class, for example, if we know a document $d$ is in class $1$ among $10$ classes in Reuter's data, to put $d$ in $S_2$, we randomly pick one class from $2$ to $10$, and mark $d$ not in that class. All Figures are put in the appendix.

First of all, we run all the algorithms on these two sample sets. We know that when sample size becomes large enough, three estimators actually convergence into the same thing, thus we take the training size relatively small. See Figure.\ref{3_m_r} and Figure.\ref{3_m_2}. According from the experiment, we can see our methods are more accurate for most of the classes, and more accurate in average.

Then we consider a more extreme case. If we have a dataset with $|S_1|=0$, that is to say, we have no positive labeled data. In this setting, traditional Naive Bayes will not work, but what will we get from our estimators? See Figure.\ref{neg_r} and Figure.\ref{neg_2}. We can see we can still get some information from negative labeled data. The accuracy is not as good as Figure.\ref{3_m_2} and Figure.\ref{3_m_r}, that is because for each of the sample, negative label is only a part of information of positive label.

At last, we test our estimator $\hat{\theta}^{L_2}$ with only $S_1$ dataset, see Figure.\ref{Cor_r} and Figure.\ref{Cor_2}. We can see our method achieve better result than traditional Naive Bayes estimator. We try to apply same training set and test the accuracy just on training set, we find traditional Naive Bayes estimator actually achieve better result, that means it might have more over-fitting problems, see Figure.\ref{over_r} and Figure.\ref{over_20}.

\section{Conclusion}
We have presented an effective learning approach with a new labeling method for partially labeled document data, for some of which we only know the sample is surely not belonging to certain classes. We encode these labels as $y_i$ or $z_i$,  and define maximum likelihood estimator $\hat{\theta}_{i_j}^{L_1}$, $\hat{\theta}_{i_j}^{L_2}$, as well as $\hat{\theta}_{i_j}^{*}$ for multinomial Naive Bayes model based on $L_1$ and $L_2$. There are several futher questions about these estimators:
\begin{enumerate}
    \item We have proved that with multinomial Naive Bayes model, our estimators have smaller variance, which means our estimators can converge to true parameters with a faster rate than the standard Naive Bayes estimator. 
    
    An interesting question is the following: if we consider a more general situation without the text classification background and the multinomial assumption, by solving the optimization problem \eqref{nb_optimal_prob} with $L_1$ and $L_2$, can we get the same conclusion with a more general chosen likelihood function $f_i$? If not, what assumption should we raise for $f_i$ to land on a similar conclusion?
    
    \item  
     The effectiveness of an algorithm in machine learning depends heavily upon well-labeled training samples; to some extent, our new estimator can utilize incorrect-labeled data or different-labeled data. Our estimator, especially $L_2$, can resolve this problem \eqref{likelihood_ratio_function}, since the incorrect-labeled data can be canceled out by the correct-labeled data, thus the partial-labeled data can still have its contribution.

    Another question is: besides $\hat{\theta}_{i_j}^{L_1}$ and $\hat{\theta}_{i_j}^{L_2}$, can we find other estimators, or even better estimators satisfying this property?
    
    \item Based on our experiment, the traditional Naive Bayes estimator acts almost perfectly in the training set as well as during the cross validation stage, but the accuracy rate in the testing set is not ideal. To quantify this observation, we are still working on a valid justification that the traditional Naive Bayes estimator has a severe over-fitting problem in the training stage.
    
\end{enumerate}

\nocite{*}

\newpage
\appendix
\section{Proof of Theorem \ref{nb_property}}

\begin{proof}
With assumption $\sum_{j=1}^v x_j = m$, we can rewrite \eqref{nb_estimator} as:
$$\hat{\theta}_{i_j} = \frac{\sum_{d\in C_i}x_j}{\sum_{d\in C_i}m}=\frac{\sum_{d\in C_i}x_j}{|C_i|m}.$$
Since $d=(x_1,x_2,...,x_v)$ is multinomial distribution, with $d$ in class $C_i$, we have: $E[x_j] = m\cdot \theta_{i_j}$, and $E[x_j^2] = m\theta_{i_j}(1-\theta_{i_j}+m\theta_{i_j}).$
\begin{enumerate}
    \item \label{nb_unbiased}
    $$\hat{\theta}_{i_j}=E[\frac{\sum_{d\in C_i}x_j}{|C_i|m}]=\frac{\sum_{d\in C_i}E[x_j]}{|C_i|m}=\frac{\sum_{d\in C_i}m\cdot \theta_{i_j}}{|C_i|m}=\theta_{i_j}.$$
    Thus $\hat{\theta}_{i_j}$ is unbiased.
    \item
    By \eqref{nb_unbiased}, we have:
    $$E[|\hat{\theta}_{i_j}-\theta_{i_j}|^2]=E[\hat{\theta}_{i_j}^2]-2\theta_{i_j}E[\hat{\theta}_{i_j}]+\theta_{i_j}^2=E[\hat{\theta}_{i_j}^2]-\theta_{i_j}^2.$$
    Then 
    \begin{equation}\label{nb_split_hat_theta}
    \hat{\theta}_{i_j}^2=\frac{(\sum_{d\in C_i}x_j)^2}{|C_i|^2m^2}=\frac{\sum_{d\in C_i}x_j^2+\sum_{d_1,d_2\in C_i}2x_j^{d_1}x_j^{d_2}}{|C_i|^2m^2},
    \end{equation}
    where $d_i=(x_1^{d_i},x_2^{d_i},...,x_v^{d_i})$ for $i=1,2$. Since:
    $$E[\sum_{d\in C_i}x_j^2]=\frac{|C_i|m\theta_{i_j}(1-\theta_{i_j}+m\theta_{i_j})}{|C_i|^2m^2}=\frac{\theta_{i_j}(1-\theta_{i_j}+m\theta_{i_j})}{|C_i|m},$$ and
    $$E[\sum_{d_1,d_2\in C_i}2x_j^{d_1}x_j^{d_2}]=\frac{|C_i|(|C_i|-1)m^2\theta_{i_j}^2}{|C_i|^2m^2}=\frac{(|C_i|-1)\theta_{i_j}^2}{|C_i|}.$$
    Plugging them into \eqref{nb_split_hat_theta} obtains:
    $$E[\hat{\theta}_{i_j}^2] = \frac{\theta_{i_j}(1-\theta_{i_j})}{|C_i|m} + \theta_{i_j}^2,$$
    thus: $E[|\hat{\theta}_{i_j}-\theta_{i_j}|^2] = \frac{\theta_{i_j}(1-\theta_{i_j})}{|C_i|m}$.
\end{enumerate}
\end{proof}

\section{Figures}

\begin{figure}[H] \centering
\subfigure[] { \label{3_m_r}
\includegraphics[width=0.47\columnwidth]{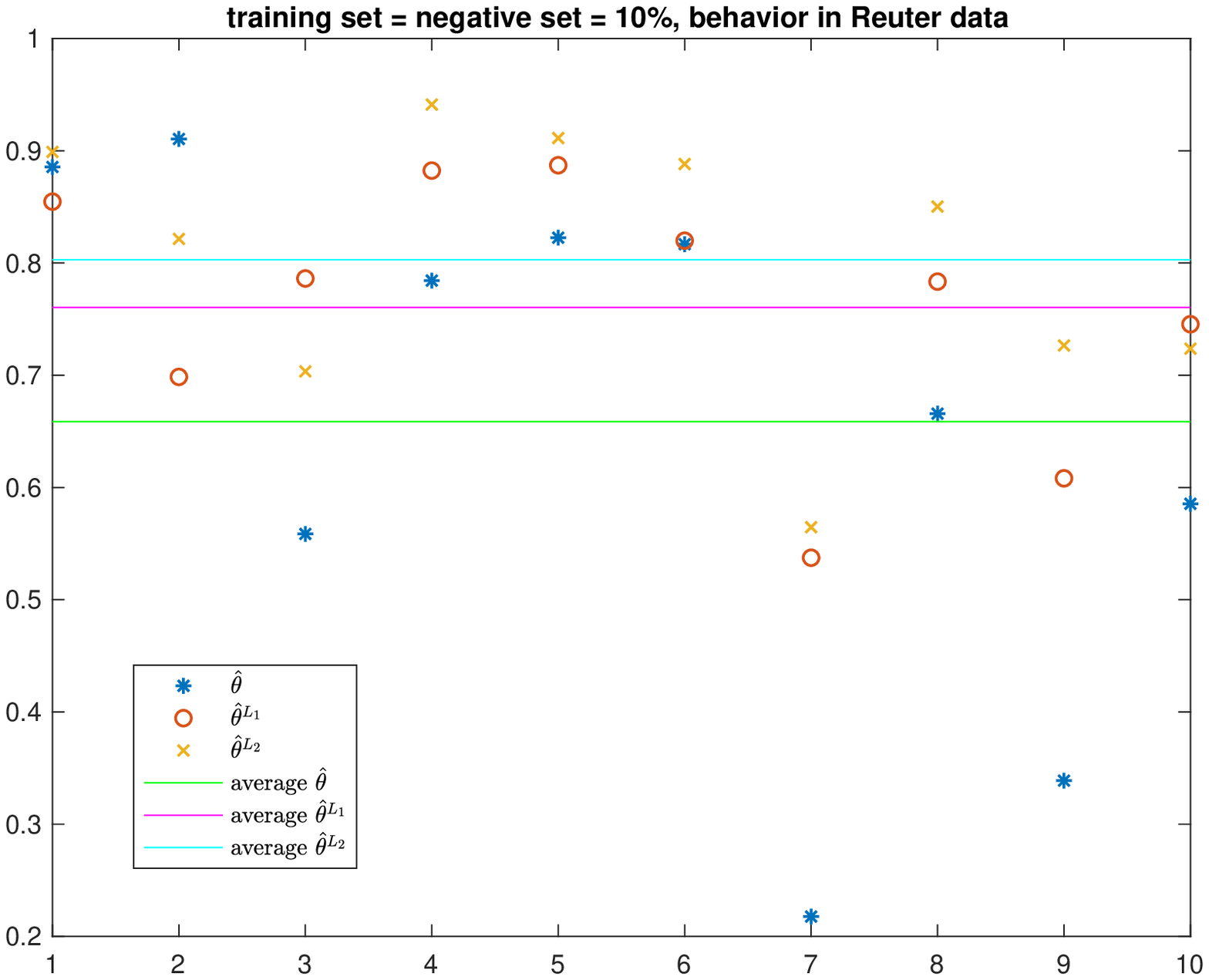} 
}
\subfigure[] { \label{3_m_2}
\includegraphics[width=0.47\columnwidth]{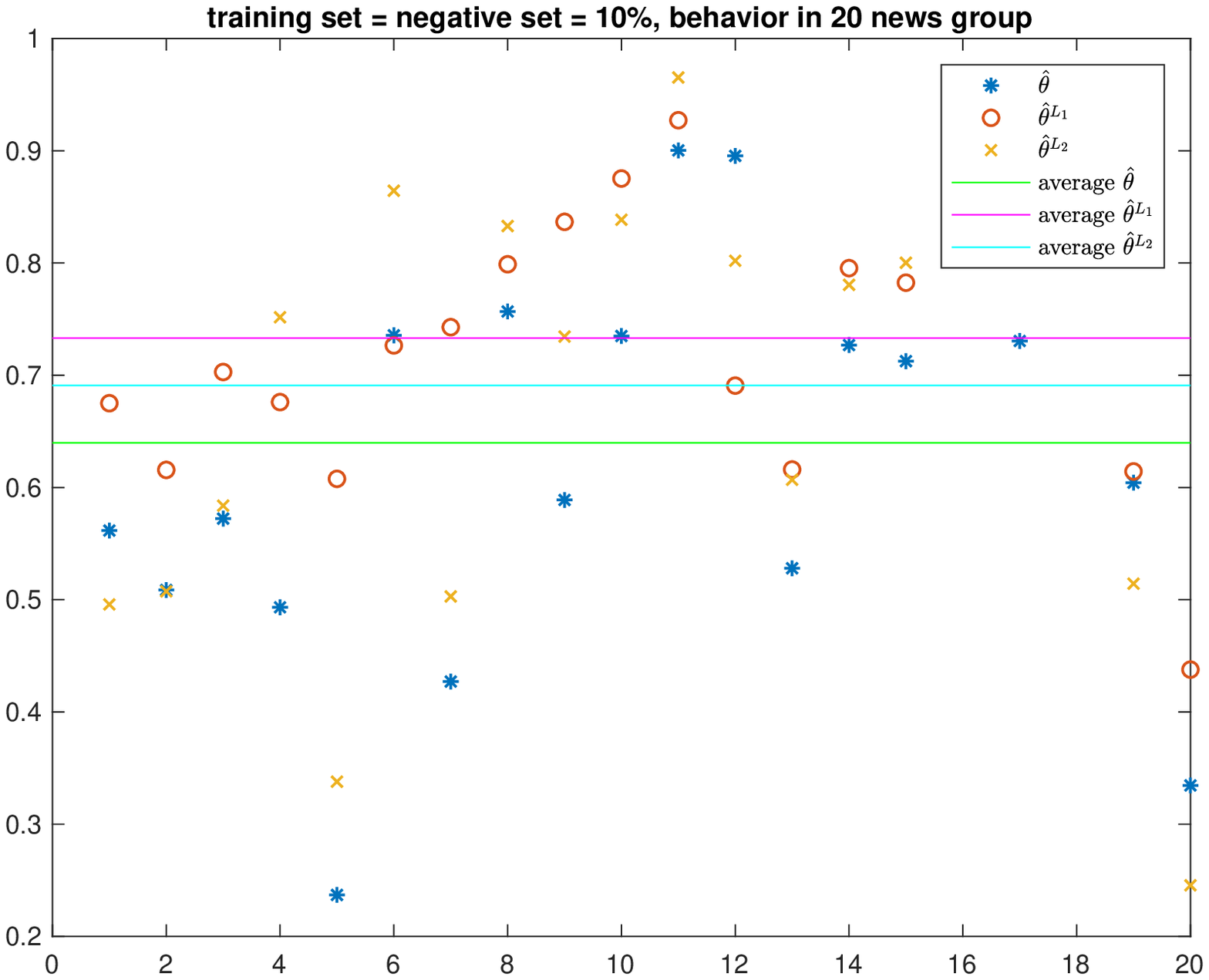} 
}
\caption{\small We take 10 largest groups in Reuter-21578 dataset (a) and 20 news group dataset (b), and take 20\% of the data as training set, among which $|S_1|=|S_2|$. The y-axis is the accuracy, and the x-axis is the class index.}
\label{3_m}
\end{figure}

\begin{figure}[H] \centering
\subfigure[]{ \label{neg_r}
  \includegraphics[width=0.47\textwidth]{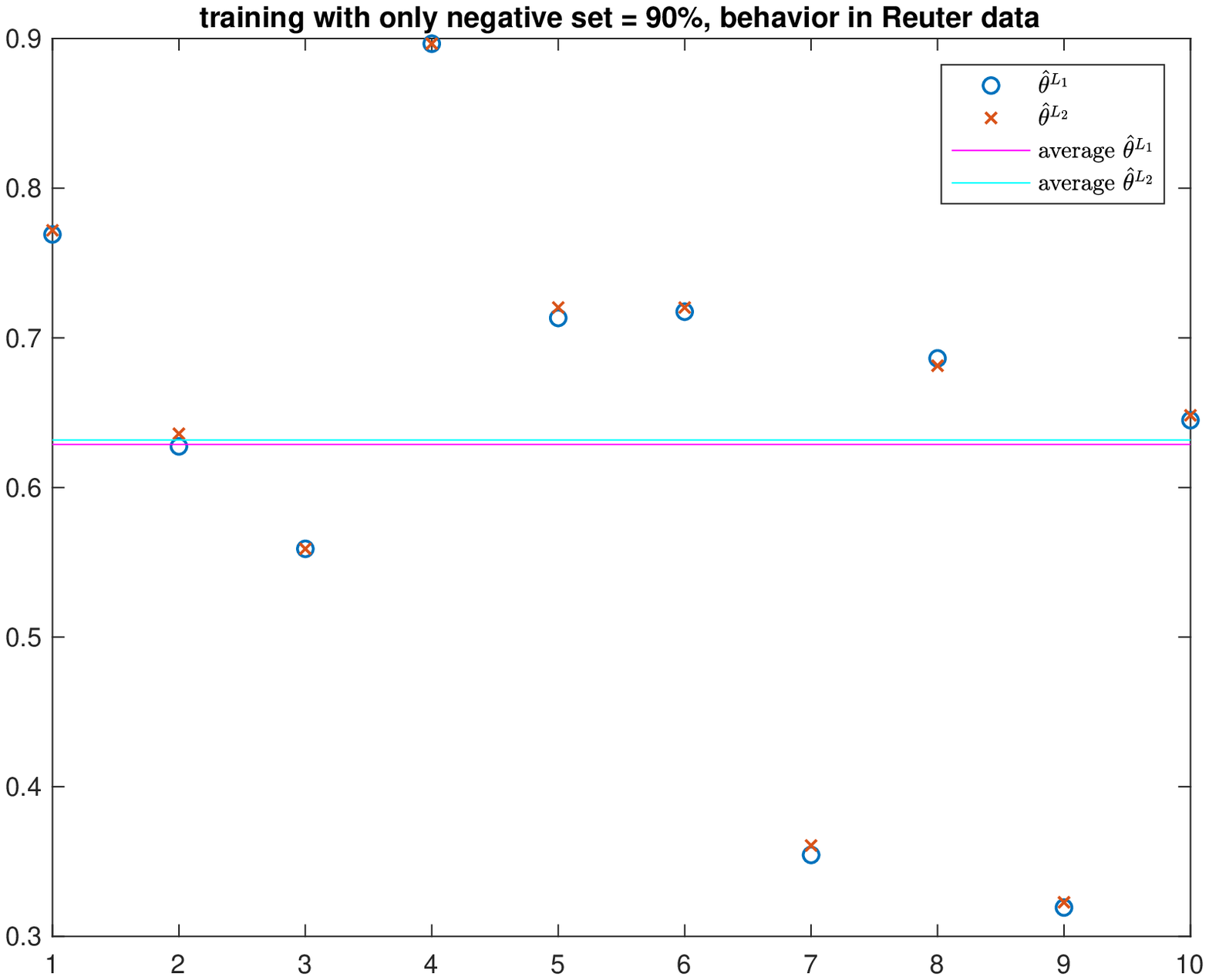}
  }
  \subfigure[]{ \label{neg_2}
  \includegraphics[width=0.47\textwidth]{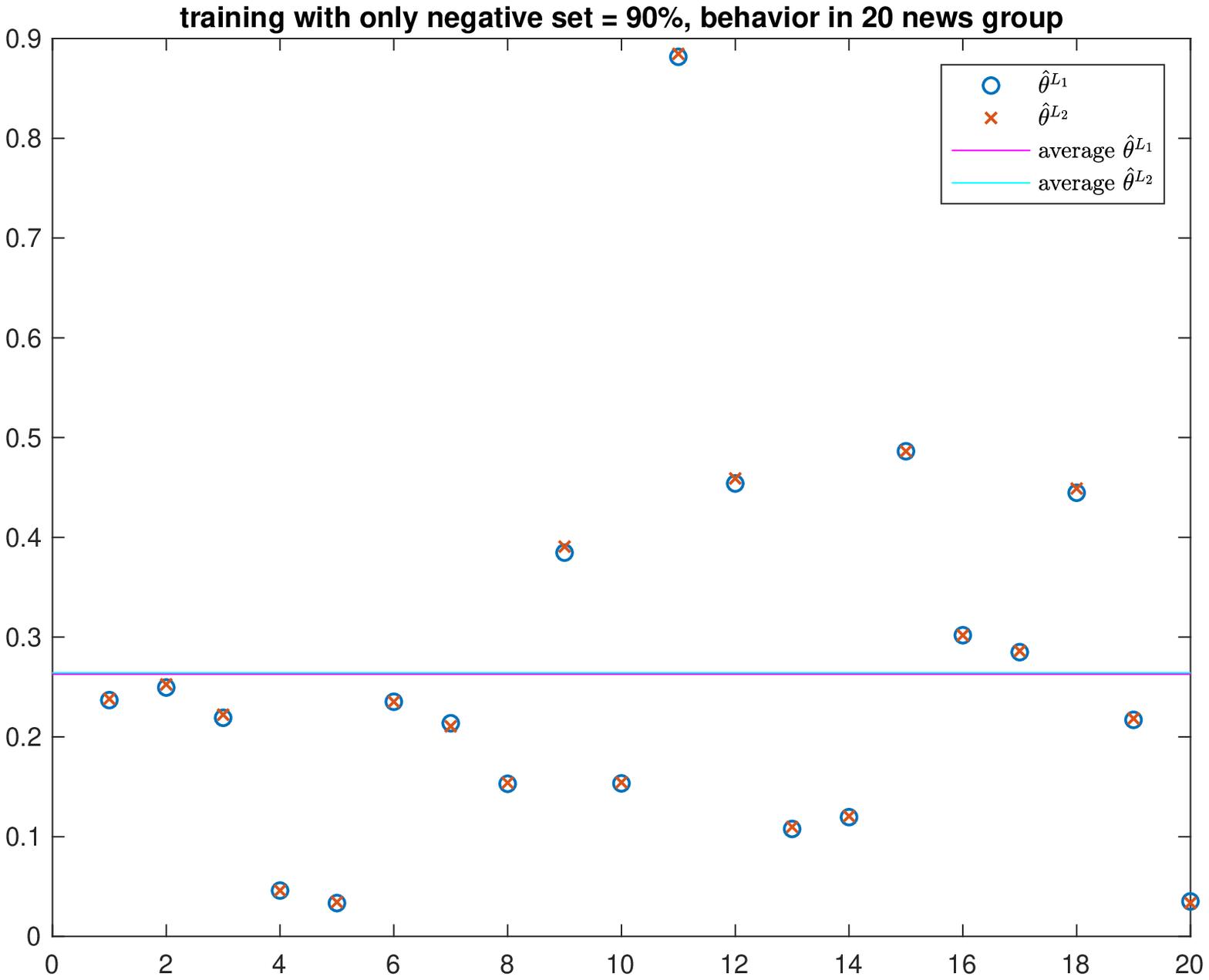}
  }
  \caption{We take 10 largest groups in Reuter-21578 dataset (a), and 20 news group dataset (b), and take 90\% of the data as $S_2$ training set. The y-axis is the accuracy, and the x-axis is the class index.}
  \label{neg}
\end{figure}

\begin{figure}[H] \centering
\subfigure[]{\label{Cor_r}
  \includegraphics[width=0.47\textwidth]{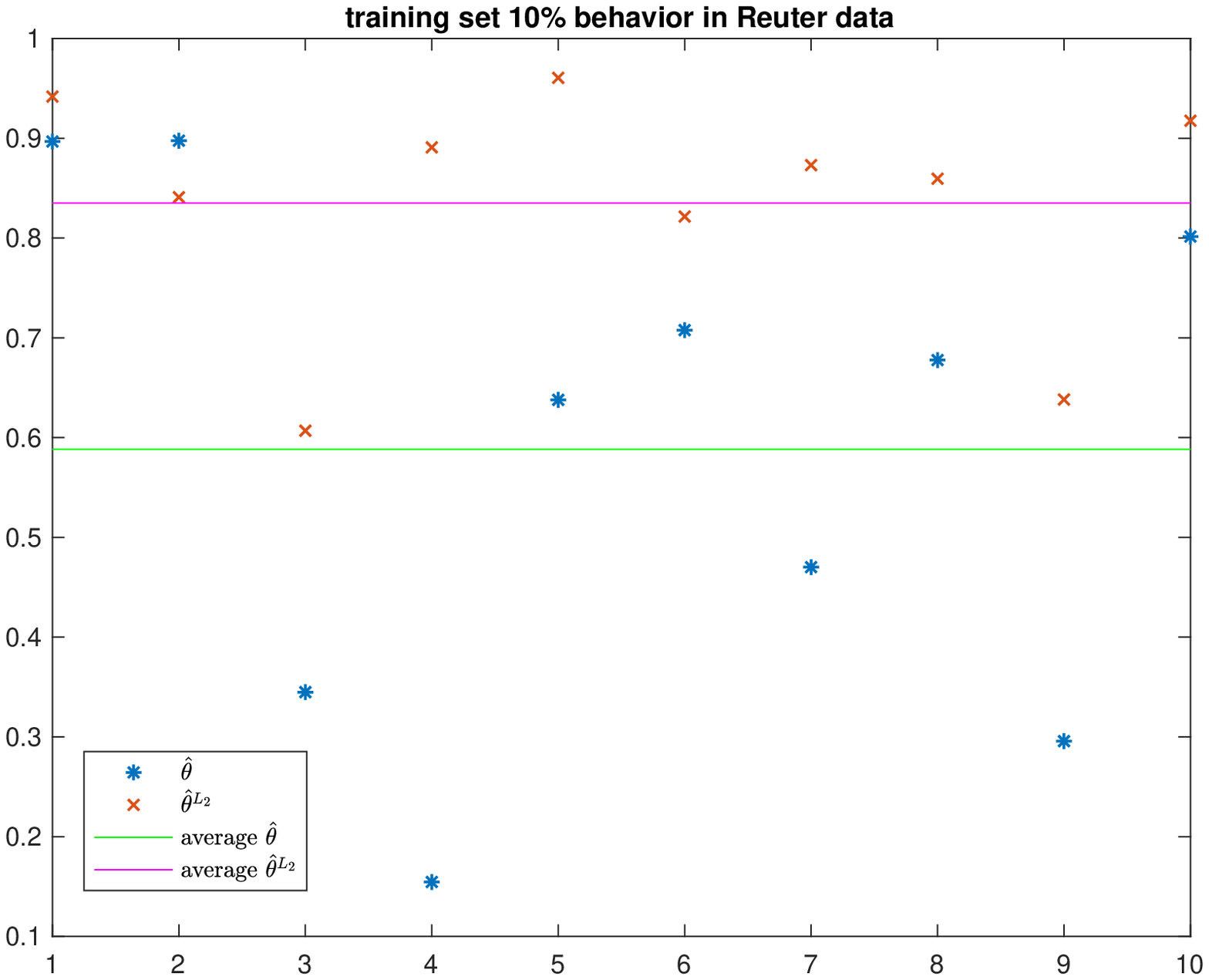}
  }
  \subfigure[]{\label{Cor_2}
  \includegraphics[width=0.47\textwidth]{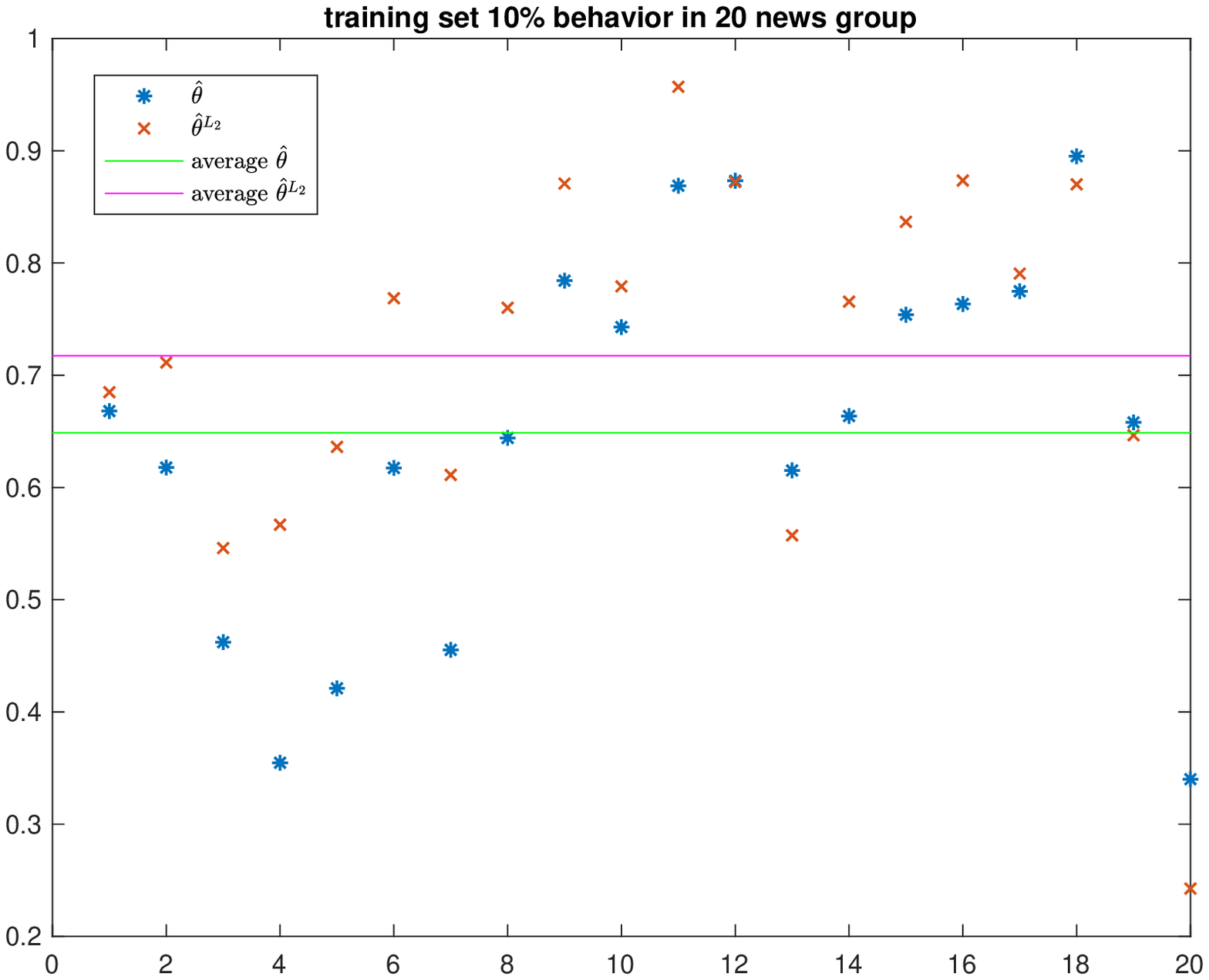}
  }
  \caption{We take 10 largest groups in Reuter-21578 dataset (a), and 20 news group dataset (b), and take 10\% of the data as $S_1$ training set. The y-axis is the accuracy, and the x-axis is the class index.}
  \label{Cor}
\end{figure}

\begin{figure}[H]\centering
\subfigure[]{ \label{over_r}
\includegraphics[width=0.47\textwidth]{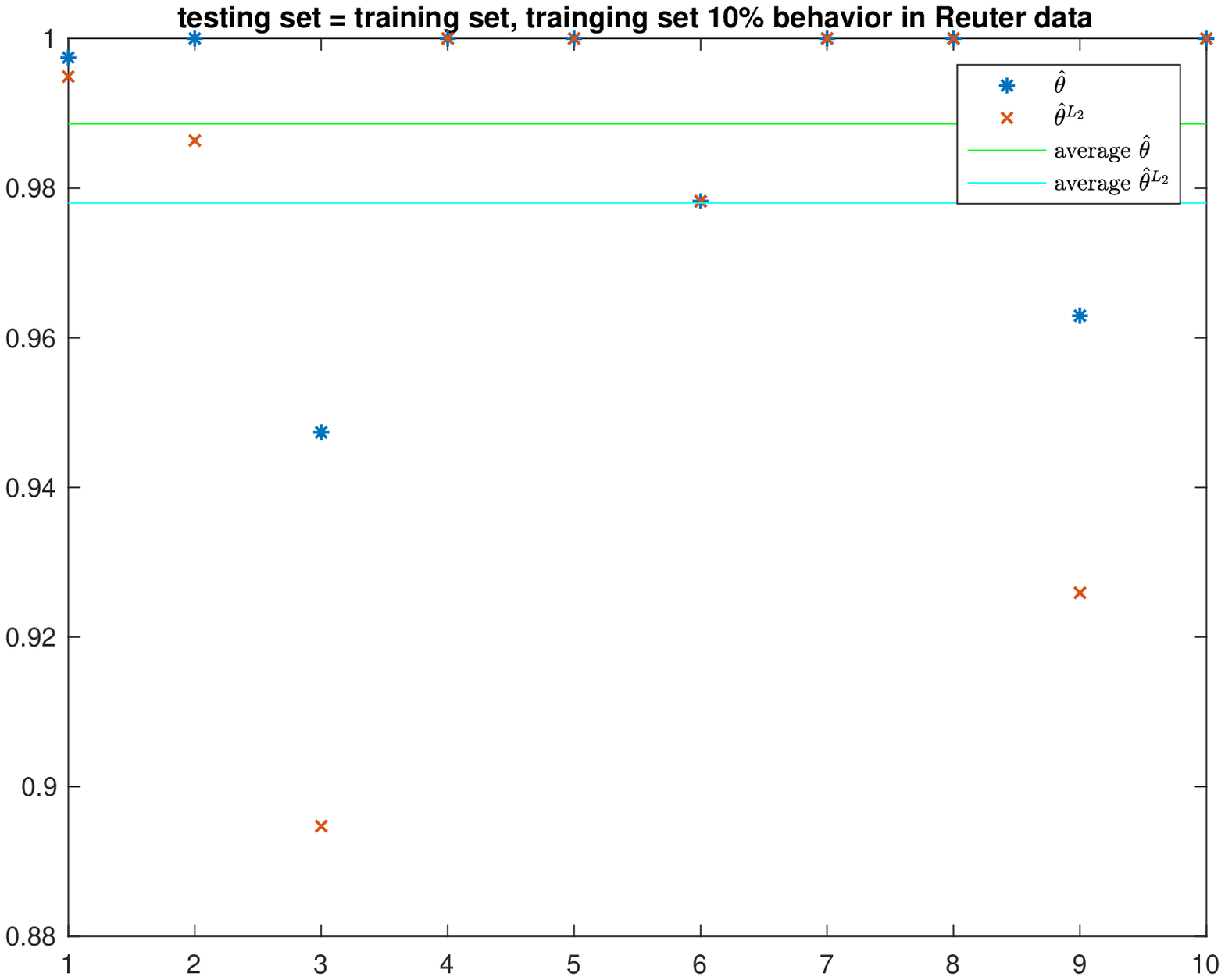}
}
\subfigure[]{ \label{over_20}
\includegraphics[width=0.47\textwidth]{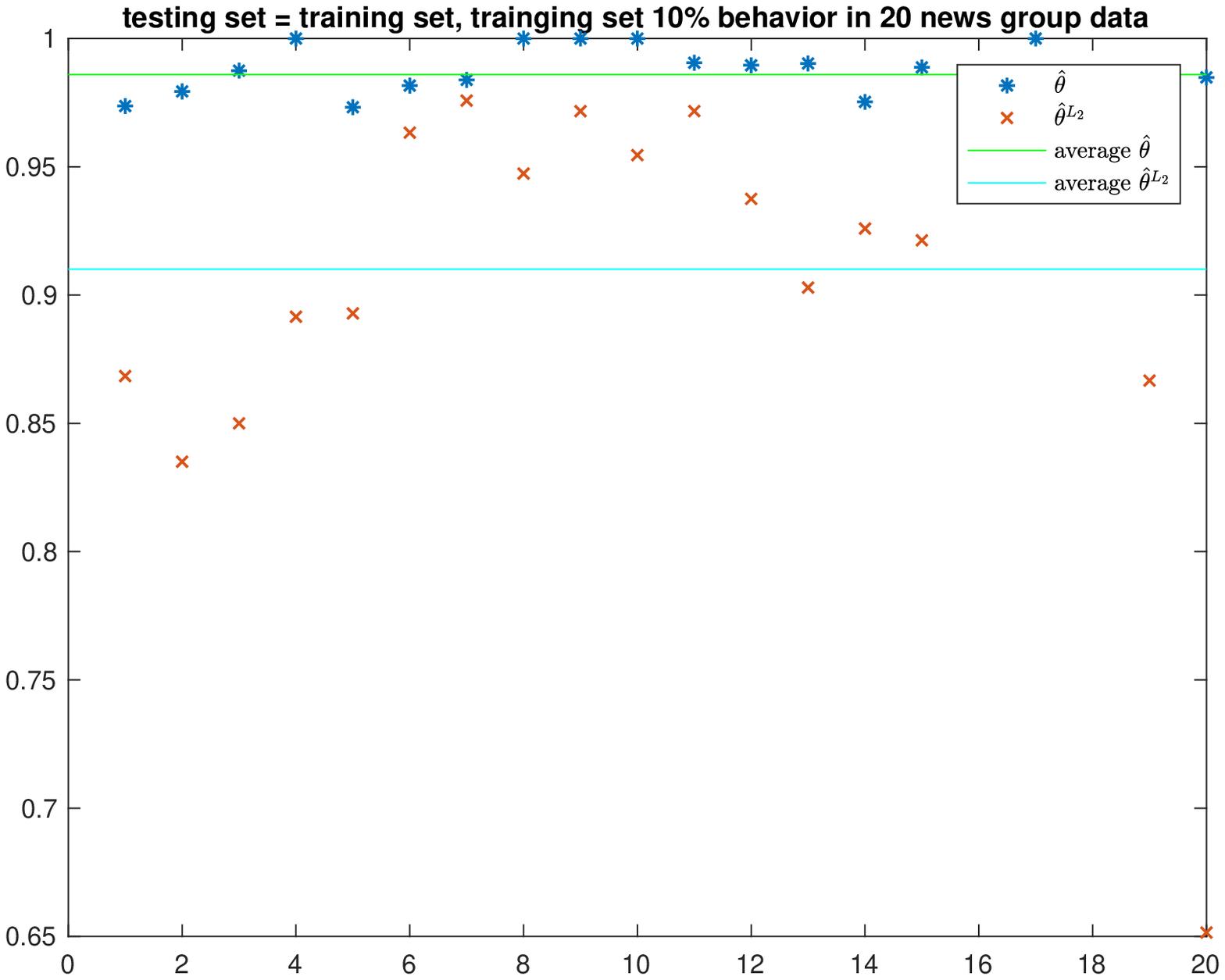}
}
\caption{We take 10 largest groups in Reuter-21578 dataset(a), and 20 news group dataset (b), and take 10\% of the data as $S_1$ training set. We test the result on training set. The y-axis is the accuracy, and the x-axis is the class index.}
 \label{over}
\end{figure}

\end{document}